\pdfoutput=1
\immediate\write18{bibtex \jobname}

\documentclass[10pt,twocolumn,letterpaper]{article}

\usepackage{cvpr}
\usepackage{amsmath}
\usepackage{amsthm}
\usepackage{txfonts}
\usepackage{blindtext}
\usepackage{epsfig}
\usepackage{graphicx}
\usepackage{hhline}
\usepackage{multirow}
\usepackage{amssymb}
\usepackage[usenames,table]{xcolor}
\usepackage[numbered,autolinebreaks]{mcode}
\definecolor{myred}{HTML}{8B0000}
\definecolor{myblue}{HTML}{000080}

\newtheorem{theorem}{Theorem}
\newtheorem{lemma}[theorem]{Lemma}
\newtheorem{corollary}{Corollary}
\newtheorem{definition}{Definition}
\usepackage{nicefrac}
\usepackage{paralist}

\newcommand{\nR}{\mathbb{R}}

\makeatletter
\newcommand*\rel@kern[1]{\kern#1\dimexpr\macc@kerna}
\newcommand*\widebar[1]{%
  \begingroup
  \def\mathaccent##1##2{%
    \rel@kern{0.8}%
    \overline{\rel@kern{-0.8}\macc@nucleus\rel@kern{0.2}}%
    \rel@kern{-0.2}%
  }%
  \macc@depth\@ne
  \let\math@bgroup\@empty \let\math@egroup\macc@set@skewchar
  \mathsurround\z@ \frozen@everymath{\mathgroup\macc@group\relax}%
  \macc@set@skewchar\relax
  \let\mathaccentV\macc@nested@a
  \macc@nested@a\relax111{#1}%
  \endgroup
}
\makeatother

\usepackage[pagebackref=true,breaklinks=true,colorlinks,
bookmarks=false,citecolor=blue,backref=false]{hyperref}

\cvprfinalcopy 

\def\cvprPaperID{1191} 

\usepackage[font=small,bf]{caption}

\DeclareMathOperator*{\pers}{p}
\definecolor{darkgreen}{rgb}{0.0588,0.4941,0.0706}

\begin{document}
\title{A Stable Multi-Scale Kernel for Topological Machine Learning}

\author{Jan Reininghaus, Stefan Huber\\
IST Austria
\and
Ulrich Bauer\\
IST Austria, 
TU M\"unchen\\
\and
Roland Kwitt\\
University of Salzburg, Austria\\
}

\maketitle
\thispagestyle{empty}


\begin{abstract}
Topological data analysis offers a rich source of valuable information to
study vision problems. Yet, so far we lack a theoretically sound connection to 
popular kernel-based learning techniques, such as kernel SVMs or kernel PCA. 
In this work, we establish such a connection by designing a  
multi-scale kernel for persistence diagrams, a stable summary 
representation of topological features in data. We show that this kernel
is positive definite and prove its stability with respect to the 
1-Wasserstein distance. Experiments on two benchmark datasets for 
3D shape classification/retrieval and texture recognition show 
considerable performance gains of the proposed method compared 
to an alternative approach that is based on the recently introduced 
persistence landscapes.
\end{abstract}


\section{Introduction}
\label{section:introduction}

\begin{figure*}
\includegraphics[width=0.98\textwidth]{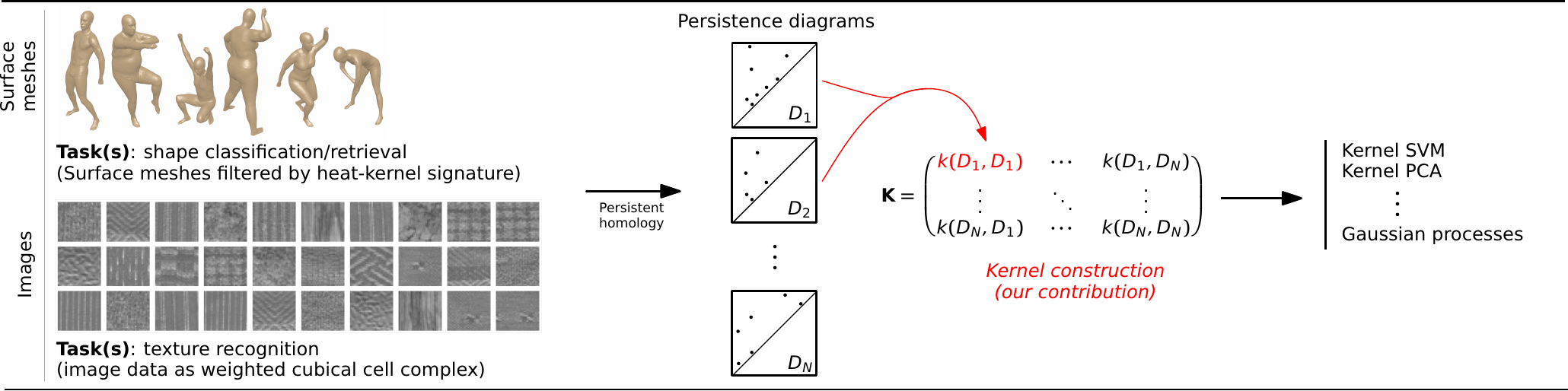}
\caption{\label{fig:motivation}
Visual data (\eg, functions on surface meshes, textures, etc.) is analyzed using 
persistent homology \cite{Edelsbrunner2010Computational}. 
Roughly speaking, persistent homology captures the birth/death times 
of topological features (\eg, connected components or holes) in the form of 
\textit{persistence diagrams}. 
Our contribution is to define a \textit{kernel for 
persistence diagrams} to enable a theoretically sound use these summary
representations in the framework of kernel-based learning techniques, popular in 
the computer vision community.}
\end{figure*}

In many computer vision problems, data (\eg, images, meshes, point 
clouds, etc.) is piped through complex processing  chains in order 
to extract information that can 
be used to address high-level inference tasks, such as recognition, 
detection or segmentation. The extracted 
information might be in the form of low-level appearance descriptors, 
\eg, SIFT \cite{Lowe04a}, or of higher-level nature, \eg, activations at specific layers of deep 
convolutional networks \cite{Krizhevsky12a}. In recognition problems, 
for instance, it is then customary to feed the consolidated data to a discriminant 
classifier such as the popular support vector machine (SVM), a kernel-based learning 
technique.

While there has been substantial progress on extracting and
encoding discriminative information, only recently have people 
started looking into the \textit{topological structure} of the 
data as an additional source of information. With the 
emergence of \textit{topological data analysis (TDA)} 
\cite{Carlsson09a}, computational tools for efficiently 
identifying topological structure have become 
readily available. Since then, several authors have
demonstrated that TDA can capture characteristics of the 
data that other methods often fail to provide, \cf~\cite{Skraba10a, Li14a}.

Along these lines, studying persistent homology \cite{Edelsbrunner2010Computational}
is a particularly popular method for TDA, since it captures
the birth and death times of topological features, \eg, connected components, 
holes, etc., at multiple scales. This information is summarized
by the \textit{persistence diagram}, a multiset of points in the plane. The key feature of persistent homology is its stability: small changes in the input data lead to small changes in the Wasserstein distance of the associated persistence diagrams~\cite{Cohen2010Lipschitz}. Considering the discrete nature of topological information, the existence of such a well-behaved summary is perhaps surprising.

Note that persistence diagrams together with the Wasserstein distance only form a metric space. Thus it is not possible to directly employ persistent homology 
in the large class of machine learning techniques that require a Hilbert space structure, like SVM or PCA. This obstacle is typically circumvented by defining a kernel function on the domain containing the data, which in turn defines a Hilbert space structure implicitly. While the Wasserstein distance itself does not naturally lead to a valid kernel (see Appendix \ref{section:definiteness}), we show that it is possible to define a kernel for persistence diagrams that is stable \wrt the 1\hbox{-}Wasserstein distance. This is the main contribution of this paper.

\textbf{Contribution.} We propose a (positive definite) 
multi-scale kernel for persistence diagrams (see Fig.~\ref{fig:motivation}). This kernel is defined via an $L_2$-valued feature map, based on ideas from scale space theory~\cite{Iijima62}. We show that our feature map is Lipschitz continuous with respect to the 1-Wasserstein distance, thereby maintaining the stability property of persistent homology.
The scale parameter of our kernel controls its robustness to noise and can be tuned to the data. We investigate, in detail, the theoretical properties of the 
kernel, and demonstrate its applicability on shape classification/retrieval 
and texture recognition benchmarks.

\section{Related work}
\label{section:relatedwork}

Methods that leverage topological information for computer 
vision or medical imaging methods can roughly be grouped into
two categories. In the first category, we identify previous work that 
\textit{directly} utilizes topological information to address a 
specific problem, such as topology-guided segmentation. 
In the second category, we identify approaches that \textit{indirectly}
use topological information. That is, information about topological
features is used as input to some machine-learning algorithm. 

As a representative of the first category, Skraba \etal 
\cite{Skraba10a} adapt the idea of persistence-based
clustering \cite{Chazal11a} in a segmentation method for 
surface meshes of 3D shapes, driven by the topological
information in the persistence diagram.
Gao \etal \cite{Gao13a} use persistence information to restore
so called \textit{handles}, \ie, topological cycles, in 
already existing segmentations of the left ventricle, extracted from
computed tomography images. In a different segmentation setup, 
Chen \etal \cite{Chen11a}
propose to directly incorporate topological constraints 
into random-field based segmentation models.

In the second category of approaches, Chung \etal \cite{Chung09a} 
and Pachauri \etal \cite{Pachauri11a} investigate the problem 
of analyzing cortical thickness measurements on 3D surface 
meshes of the human cortex in order to study developmental and neurological 
disorders. In contrast to \cite{Skraba10a}, persistence
information is not used directly, but rather as a \textit{descriptor} 
that is fed to a discriminant classifier in order to distinguish 
between normal control patients and patients with Alzheimer's disease/autism. Yet, the step of training the classifier with topological
information is typically done in a rather adhoc manner. 
In \cite{Pachauri11a} for instance, the persistence diagram is 
first rasterized on a regular grid, then a kernel-density estimate is 
computed, and eventually the vectorized discrete probability density function is used as a 
feature vector to train a SVM using standard kernels for $\mathbb{R}^n$. It is however unclear 
how the resulting kernel-induced distance behaves with respect to existing 
metrics (\eg, bottleneck or Wasserstein distance) and how properties 
such as stability are affected. An approach that directly uses well-established distances between persistence diagrams for recognition 
was recently proposed by Li \etal \cite{Li14a}. Besides bottleneck and Wasserstein distance, the authors employ persistence landscapes \cite{Bubenik13a} and the corresponding distance in their experiments.
Their results expose the complementary nature of persistence 
information when combined with traditional bag-of-feature approaches.
While our empirical study in Sec.~\ref{subsection:empirical_results} is inspired by \cite{Li14a}, we primarily focus on the development of 
the kernel; the combination with other methods is straightforward.
  
In order to enable the use of persistence information in 
machine learning setups, Adcock \etal \cite{Adcock13} propose to 
compare persistence diagrams using a feature vector motivated by 
algebraic geometry and invariant theory. The features are defined 
using algebraic functions of the birth and death values in the 
persistence diagram.

From a conceptual point of view, Bubenik's concept of \textit{persistence landscapes} \cite{Bubenik13a} is probably the closest to ours, being another kind of feature map for persistence diagrams. While persistence landscapes were not explicitly designed for use in machine learning algorithms, we will draw the connection to our work in Sec.~\ref{subsection:landscape_comparison} and show that they in fact admit the definition of a valid positive definite kernel. Moreover, both
persistence landscapes as well as our approach represent computationally 
attractive alternatives to the bottleneck or Wasserstein distance, which
both require the solution of a matching problem. 

\section{Background}
\label{section:background}

First, we review some fundamental notions and results from persistent homology that will be relevant for our work. 
\paragraph{Persistence diagrams.}
\emph{Persistence diagrams} are a concise description of the topological changes occuring in a growing sequence of shapes, called \emph{filtration}.
In particular, during the growth of a shape, holes of different dimension (\ie, gaps between components, tunnels, voids, etc.) may appear and disappear. Intuitively, a $k$-dimensional hole, born at time~$b$ and filled at time~$d$, gives rise to a point~$(b,d)$ in the $k$\textsuperscript{th} persistence diagram.
A persistence diagram is thus a multiset of points in $\mathbb{R}^2$.
Formally, the persistence diagram is defined using a standard concept from algebraic topology called \emph{homology}; see \cite{Edelsbrunner2010Computational} for details.

Note that not every hole has to disappear in a filtration. 
Such holes give rise to \emph{essential} features and are 
naturally represented by points of the form $(b,\infty)$ 
in the diagram. Essential features therefore capture the topology of
the final shape in the filtration. 
In the present work, we do not consider these
features as part of the persistence diagram.
Moreover, all persistence diagrams will be assumed to be finite,
as is usually the case for persistence diagrams coming from data.


\begin{figure}[t!]
\centering
\includegraphics[width=0.75\columnwidth]{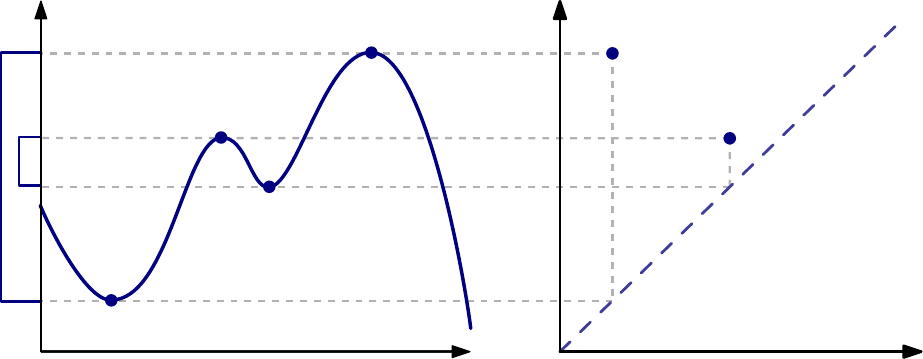}
\caption{A function $\mathbb{R}\to\mathbb{R}$ (left) and its 0\textsuperscript{th} persistence diagram (right). Local minima create a connected component in the corresponding sublevel set, while local maxima merge connected components. The pairing of birth and death is shown in the persistence diagram.\label{fig:persistence-dgm}}
\end{figure}

\vspace{-0.2cm}
\paragraph{Filtrations from functions.}
A standard way of obtaining a filtration is to consider the \emph{sublevel sets} $f^{-1}(-\infty,t]$ of a function $f\colon\Omega\to\mathbb R$ defined on some domain~$\Omega$, for $t\in\mathbb R$. It is easy to see that the sublevel sets indeed form a filtration parametrized by $t$. We denote the resulting persistence diagram by $D_f$; see Fig.~\ref{fig:persistence-dgm} for an illustration.

As an example, consider a grayscale image, where $\Omega$ is the rectangular domain of the image and $f$ is the grayscale value at any point of the domain (\ie, at a particular pixel). A sublevel set would thus consist of all pixels of $\Omega$ with value up to a certain threshold~$t$.
Another example would be a piecewise linear function on a triangular mesh $\Omega$, such as the
popular heat kernel signature \cite{Sun09a}.
Yet another commonly used filtration arises from point clouds~$P$ embedded in $\mathbb R^n$, by considering the distance function $d_P(x)=\min_{p\in P}\|x-p\|$ on $\Omega=\mathbb R^n$. The sublevel sets of this function are unions of balls around $P$. Computationally, they are usually replaced by equivalent constructions called \emph{alpha shapes}.

\vspace{-0.1cm}
\paragraph{Stability.}
A crucial aspect of the persistence diagram $D_f$ of a function $f$ is its stability with respect to 
perturbations of $f$. In fact, only stability guarantees that one can infer information about the 
function $f$ from its persistence diagram $D_f$ in the presence of noise.

Formally, we consider $f \mapsto D_f$ as a map of metric spaces and define \emph{stability} as Lipschitz continuity of this map. This requires choices of metrics both on the set of functions and the set of persistence diagrams. For the functions, the $L_\infty$ metric is commonly used.

There is a natural metric associated to persistence diagrams, called the \emph{bottleneck distance.}
Loosely speaking, the distance of two diagrams is expressed by minimizing the largest distance of any two corresponding points, over all bijections between the two diagrams.
Formally, let $F$ and $G$ be two persistence diagrams, each augmented by adding each point $(t,t)$ on the diagonal with countably infinite multiplicity. The \emph {bottleneck distance} is 
\begin{equation}
d_B(F,G)=\inf_\gamma\sup_{x\in F}\|x-\gamma(x)\|_\infty ,
\label{eqn:bottleneck_distance}
\end{equation}
where $\gamma$ ranges over all bijections from the individual points of\/~$F$ to the individual 
points of\/~$G$. 
As shown by Cohen-Steiner \etal.~\cite{Steiner07a}, persistence diagrams are stable with respect to the bottleneck distance.

%

The bottleneck distance embeds into a more general class of distances, called \emph{Wasserstein distances}. For any positive real number $p$, the \emph{$p$-Wasserstein distance} is 
\begin{equation}
d_{W,p}(F, G)=\left(\inf_\gamma\sum_{x\in F}\|x-\gamma(x)\|_\infty^p\right)^{1\over p},
\label{eqn:wasserstein_distance}
\end{equation}
where again $\gamma$ ranges over all bijections from the individual elements of
$F$ to the individual elements of $G$.  Note that taking the limit $p\to\infty$
yields the bottleneck distance, and we therefore define $d_{W,\infty} = d_B$.
We have the following result bounding the $p$-Wasserstein distance in terms of
the $L_\infty$ distance:

\begin{theorem}[Cohen-Steiner et al.~\cite{Cohen2010Lipschitz}]
Assume that $X$ is a compact triangulable metric space such that for every 1-Lipschitz function $f$ on $X$ and for $k\geq 1$, the \emph{degree $k$ total persistence} $\sum_{(b,d)\in D_f}(d-b)^k$ is bounded above by some constant $C$.
Let $f,g$ be two $L$-Lipschitz piecewise linear functions on $X$. Then for all $p\geq k$,
\begin{equation}
d_{W,p}(D_f,D_g) \leq (LC)^{1\over p} \|f-g\|_\infty^{1-\frac k p}.
\label{eqn:ch_wasserstein_l_inf_bound}
\end{equation}
\end{theorem}
We note that, strictly speaking, this is not a stability result in the sense of Lipschitz continuity, since it only establishes H\"older continuity. Moreover, it only gives a constant upper bound for the Wasserstein distance when $p=1$. 


\paragraph{Kernels.}

Given a set $\mathcal{X}$, a function $k \colon \mathcal{X} \times \mathcal{X}
\to \mathbb{R}$ is a \emph{kernel} if there exists a Hilbert space
$\mathcal{H}$, called \emph{feature space}, and a map $\Phi \colon \mathcal{X}
\to \mathcal{H}$, called \emph{feature map}, such that $k(x,y) = \langle
\Phi(x), \Phi(y) \rangle_{\mathcal{H}}$ for all $x, y \in \mathcal{X}$.
Equivalently, $k$ is a kernel if it is symmetric and positive
definite~\cite{Scholkopf01}.
Kernels allow to apply machine learning algorithms operating on a Hilbert space
to be applied to more general settings, such as strings, graphs, or, in our case,
persistence diagrams.

A kernel induces a pseudometric $d_k(x,y) = (k(x,x) + k(y,y) - 2\,
k(x,y))^{\nicefrac{1}{2}}$ on $\mathcal{X}$, which is the distance $\|\Phi(x) -
\Phi(y)\|_\mathcal{H}$ in the feature space.
We call the kernel $k$ \emph{stable} \wrt a metric $d$ on $\mathcal{X}$ if
there is a constant $C > 0$ such that $d_k(x,y) \le C \, d(x,y)$ for all $x, y
\in \mathcal{X}$. Note that this is equivalent to Lipschitz continuity of the
feature map.

The stability of a kernel is particularly useful for classification 
problems: assume that there exists a separating hyperplane $H$ for two classes 
of data points with margin $m$. If the data points are perturbed by some 
$\epsilon < m/2$, then $H$ still separates the two classes with a margin 
$m - 2\epsilon$.

\vspace{-0.1cm}
\section{The persistence scale-space kernel}
\label{section:kernel}

We propose a stable \textit{multi-scale} kernel $k_\sigma$ for the set of
persistence diagrams $\mathcal{D}$. This kernel will be defined via a feature
map $\Phi_\sigma: \mathcal{D} \rightarrow L_2(\Omega)$, with $\Omega \subset
\nR^2$ denoting the closed half plane above the diagonal.

To motivate the definition of $\Phi_\sigma$, we point out that the set of persistence
diagrams, \ie, multisets of points in $\mathbb{R}^2$, does not possess a
Hilbert space structure per se. However, a persistence diagram $D$ can be uniquely 
represented as a sum of Dirac delta distributions%
\footnote{A Dirac delta distribution is a functional that evaluates a given
    smooth function at a point.}%
, one for each
point in $D$. Since Dirac deltas are functionals in the Hilbert space
$H^{-2}(\mathbb{R}^2)$ \cite[Chapter 7]{Iorio01}, we can embed the set of persistence diagrams
into a Hilbert space by adopting this point of view.

Unfortunately, the induced metric on $\mathcal{D}$ does
\textit{not} take into account the distance of the points to the diagonal, and
therefore cannot be robust against perturbations of the diagrams. 
Motivated by scale-space
theory~\cite{Iijima62}, we address this issue by using the sum of Dirac deltas
as an initial condition for a heat diffusion problem with a Dirichlet boundary
condition on the diagonal.
The solution of this partial differential equation is an $L_2(\Omega)$ function
for any chosen scale parameter $\sigma>0$. In the following paragraphs, we will
\begin{compactenum}[1)]
\item define the persistence scale space kernel $k_\sigma$, \item derive a simple
formula for evaluating $k_\sigma$, and \item prove stability of $k_\sigma$
\wrt the $1$-Wasserstein distance.
\end{compactenum}


\begin{definition}
Let $\Omega = \{ x = (x_1, x_2) \in \mathbb{R}^2\colon x_2 \geq x_1 \}$ denote
the space above the diagonal, and let $\delta_p$ denote a Dirac delta centered at
the point $p$.  For a given persistence diagram $D$, we now consider the
solution $u\colon \Omega \times \mathbb{R}_{\geq 0} \rightarrow \mathbb{R},
(x,t) \mapsto u(x,t)$ of the partial differential equation\footnote{Since the
    initial condition~\eqref{eq:initial_condition} is not an $L_2(\Omega)$
    function, this equation is to be understood in the sense of distributions. For a
    rigorous treatment of existence and uniqueness of the solution, see
    \cite[Chapter
    7]{Iorio01}.}
\begin{align}
    \Delta_x u &= \partial_t u &&\text{in $\Omega \times \mathbb{R}_{> 0}$}, \\
    u &= 0 &&\text{on $\partial\Omega \times \mathbb{R}_{\geq 0}$}, \\
    u &= \sum_{p \in D} \delta_p &&\text{on $\Omega \times \{0\}$} \label{eq:initial_condition}.
\end{align}
The feature map $\Phi_\sigma \colon \mathcal{D} \to L_2(\Omega)$ at scale
$\sigma > 0$ of a persistence diagram $D$ is now defined as $\Phi_\sigma(D) =
\left.u\right|_{t=\sigma}$.  This map yields the persistence scale 
space kernel $k_\sigma$ on $\mathcal{D}$ as
\begin{equation}
    k_\sigma(F,G) = \langle \Phi_\sigma(F),\Phi_\sigma(G) \rangle_{L_2(\Omega)}.
	\label{eq:kernel_definition}
\end{equation}


\end{definition}

Note that $\Phi_\sigma(D)=0$ for some $\sigma>0$ implies that $u=0$  on $\Omega \times \{0\}$,
which means that $D$ has to be the empty diagram. From linearity of the solution operator it now follows that 
$\Phi_\sigma$ is an injective map.

The solution of the partial differential equation can be obtained by extending 
the domain from  $\Omega$ to $\mathbb{R}^2$ and replacing \eqref{eq:initial_condition} with 
\begin{align}
u &= \sum_{p \in D} \delta_p - \delta_{\overline{p}} &&\text{on $\mathbb{R}^2 \times \{0\}$,}
\label{eq:mod_initial_condition}
\end{align}
where $\overline{p}=(b,a)$ is $p=(a,b)$ mirrored at the diagonal. It can be shown that restricting the solution of this extended problem to $\Omega$ yields a solution for the original equation. It is given by convolving the initial condition \eqref{eq:mod_initial_condition} with a Gaussian kernel:
\begin{align}
    \label{eq:solutionpde}
    u(x, t) = \frac{1}{4\pi t} \sum_{p \in D} e^{-\frac{\|x - p\|^2}{4t}} -
    e^{-\frac{\|x - \overline{p}\|^2}{4t}}.
\end{align}
Using this closed form solution of $u$, we can derive a simple expression for evaluating the kernel explicitly:
\begin{align}
k_\sigma(F,G) 
&= \frac{1}{8 \pi \sigma} \sum_{\substack{p \in F\\q \in G}} e^{-\frac{\|p-q\|^2}{8\sigma}} - e^{-\frac{\|p-\overline{q}\|^2}{8\sigma}}.
\label{eqn:l2ip}
\end{align}
We refer to Appendix~\ref{section:kclosedform} for the 
elementary derivation of \eqref{eqn:l2ip} and for a 
visualization (see Appendix~\ref{section:featuremapplots}) of the solution \eqref{eq:solutionpde}. 
Note that the kernel can be computed in 
$\mathcal{O}(|F| \cdot |G|)$ time, where $|F|$ and $|G|$ 
denote the cardinality of the multisets $F$ and $G$,
respectively.





\begin{theorem}
\label{thm:robustness}
The kernel $k_\sigma$ is $1$-Wasserstein stable.
\end{theorem}

\begin{proof}
To prove $1$-Wasserstein stability of $k_\sigma$, we show Lipschitz continuity of
the feature map $\Phi_\sigma$ as follows:
\begin{equation}
    \|\Phi_\sigma(F) - \Phi_\sigma(G)\|_{L_2(\Omega)} \leq
    \frac{1}{\sigma\sqrt{8\pi}}\,d_{W,1}(F,G),
    \label{eq:robustness}
\end{equation}
where $F$ and $G$ denote persistence diagrams that have been augmented with
points on the diagonal. Note that augmenting diagrams with points on the
diagonal does not change the values of $\Phi_\sigma$, as can be seen from
\eqref{eq:solutionpde}.
%
%
Since the unaugmented persistence diagrams are assumed to be finite,
some matching $\gamma$ between $F$ and $G$ achieves
the infimum in the definition of the Wasserstein distance, 
    $d_{W,1}(F,G) = \sum_{u \in F} \|u-\gamma(u)\|_\infty$.  Writing $N_u(x) = \frac{1}{4\pi
        \sigma} e^{-\frac{\|x - u\|^2_2}{4\sigma}}$, we have
    $\|N_u-N_v\|_{L_2(\mathbb{R}^2)} = \frac{1}{\sqrt{4 \pi \sigma}} \cdot
    \sqrt{1- e^{-\frac{\|u-v\|_2^2}{8 \sigma}}}$. The Minkowski inequality and
    the inequality $e^{-\xi} \ge 1 - \xi$ finally yield
\begin{align*}
    &\|\Phi_\sigma(F) - \Phi_\sigma(G)\|_{L_2(\Omega)} \\
    &\le \left\| \sum_{u \in F} (N_u - N_{\overline{u}}) - (N_{\gamma(u)} -
        N_{\overline{{\gamma(u)}}}) \right\|_{L_2(\nR^2)} \\
    &\le 2 \sum_{u \in F} \| N_u - N_{\gamma(u)} \|_{L_2(\nR^2)} \\
    &\le \frac{1}{\sqrt{\pi \sigma}} \sum_{u \in F} \sqrt{1-
        e^{-\frac{\|u-{\gamma(u)}\|_2^2}{8 \sigma}}} \\
    &\le \frac{1}{\sigma \sqrt{8 \pi}} \sum_{u \in F} \|u-{\gamma(u)}\|_2 \quad
        \le \quad \frac{1}{2\sigma \, \sqrt{\pi}} d_{W,1}(F,G) . \qedhere
\end{align*}
\end{proof}

We refer to the left-hand side of \eqref{eq:robustness} as the
\textit{persistence scale space distance} $d_{k_{\sigma}}$ between $F$ and $G$. Note
that the right hand side of \eqref{eq:robustness} decreases as $\sigma$
increases. Adjusting~$\sigma$ accordingly allows to counteract the influence of noise in the input data,
which causes an increase in $d_{W,1}(F,G)$.
We will see in Sec.~\ref{subsection:texture_recognition} that tuning 
$\sigma$ to the data can be beneficial for the overall performance of machine learning methods.

A natural question arising from Theorem~\ref{thm:robustness} is
whether our stability result extends to $p>1$. To answer this question,
we first note that our kernel is \textit{additive:} we call 
a kernel $k$ on persistence diagrams additive if $k(E \cup F, G) = 
k(E, G) + k(F, G)$ for all $E, F, G \in \mathcal{D}$. By choosing 
$F = \emptyset$, we see that if $k$ is additive then $k(\emptyset, G)= 0$ 
for all $G \in \mathcal{D}$. We further say that a 
kernel $k$ is \emph{trivial} if $k(F, G) = 0$ for
all $F,G \in \mathcal{D}$. The next theorem establishes that Theorem~\ref{thm:robustness} 
is sharp in the sense that \textit{no} non-trivial additive kernel can be stable 
\wrt the $p$-Wasserstein distance when $p > 1$.




\begin{theorem}
    A non-trivial additive kernel $k$ on persistence diagrams is not stable
    \wrt $d_{W,p}$ for any $1 < p \leq \infty$.
\end{theorem}

\begin{proof}
By the non-triviality of $k$, it can be shown that there exists an $F \in \mathcal{D}$ 
such that $k(F, F) > 0$. We prove the
claim by comparing the rates of growth of $d_{k_\sigma}(\bigcup_{i=1}^n F,
\emptyset)$ and $d_{W,p}(\bigcup_{i=1}^n F, \emptyset)$ \wrt $n$. We have
\[     d_{k_\sigma}\left(\bigcup_{i=1}^n F, \emptyset\right) = n \, \sqrt{k(F,F)}. \]
On the other hand,
\[     d_{W,p}\left(\bigcup_{i=1}^n F, \emptyset\right) =
    d_{W,p}(F, \emptyset) \cdot
    \begin{cases}
        \sqrt[p]{n} & \text{if $p < \infty$} ,\\
        1 & \text{if $p = \infty$} .
    \end{cases}
\]
Hence, $d_{k_\sigma}$ can not be bounded by $C \cdot d_{W,p}$ with a
constant $C > 0$ if $p > 1$.
\end{proof}

\section{Evaluation}
\label{section:evaluation}

To evaluate the kernel proposed in Sec.~\ref{section:kernel}, we
investigate conceptual differences to persistence landscapes
in Sec.~\ref{subsection:landscape_comparison}, 
and then consider its performance in the context of shape 
classification/retrieval and texture recognition 
in Sec.~\ref{subsection:empirical_results}. 

\subsection{Comparison to persistence landscapes}
\label{subsection:landscape_comparison}

In \cite{Bubenik13a}, Bubenik introduced \textit{persistence landscapes},
a representation of persistence diagrams as functions in the Banach
space $L_p(\mathbb{R}^2)$. This construction was mainly intended 
for statistical computations, enabled by the vector space structure of $L_p$. 
For $p=2$, we can use the Hilbert space structure of 
$L_2(\mathbb{R}^2)$ to construct a kernel analogously to \eqref{eq:kernel_definition}. 
For the purpose of this work, we refer to this kernel as the \textit{persistence landscape kernel} 
$k^L$ and denote by $\Phi^L\colon \mathcal{D} \to L_2(\mathbb{R}^2)$ the corresponding feature 
map. The kernel-induced distance is denoted by $d_{k^L}$. Bubenik
shows stability \wrt a weighted version
of the Wasserstein distance, which for $p=2$ can be summarized
as:

\begin{theorem}[Bubenik \cite{Bubenik13a}]
\label{thm:landscape_stability}
For any two persistence diagrams $F$ and $G$ we have
\begin{align}
\begin{split}
& \|\Phi^L(F)  - \Phi^L(G)\|_{L_2(\mathbb{R}^2)} \leq \\
& \inf_{\gamma} \left( 
\sum_{u \in F} \pers(u) \|u-\gamma(u)\|^2_\infty + \frac{2}{3} \|u-\gamma(u)\|^3_\infty
 \right)^{\frac{1}{2}} ,
\end{split}
\label{eqn:landscape_stability}
\end{align}
where $\pers(u)=d-b$ denotes the persistence of $u=(b,d)$, and $\gamma$ ranges over all bijections from $F$ to $G$.
\end{theorem}

For a better understanding of the stability results given in Theorems \ref{thm:robustness} and 
\ref{thm:landscape_stability}, we present and discuss two thought experiments.

For the first experiment, let $F_{\lambda} = \{-\lambda,\lambda\}$ and 
$G_{\lambda} = \{-\lambda+1,\lambda+1\}$ be two diagrams with
one point each and $\lambda \in \mathbb{R}_{\geq 0}$. The two points 
move away from the diagonal with increasing 
$\lambda$, while maintaining the same Euclidean distance to each other. Consequently,
$d_{W,p}(F_\lambda,G_\lambda)$ and $d_{k_\sigma}(F_\lambda,G_\lambda)$ asymptotically 
approach a constant  as $\lambda\to \infty$.
In contrast, $d_{k^L}(F_\lambda,G_\lambda)$ grows in the order of $\sqrt{\lambda}$ and, 
in particular, 
is unbounded. This means that $d_{k^L}$ emphasizes points of high persistence
in the diagrams, as reflected by the weighting term $\pers(u)$ in \eqref{eqn:landscape_stability}.

In the second experiment, we compare persistence diagrams
from data samples of two fictive classes A (\ie, $F$,$F'$) and B (\ie, $G$), 
illustrated in Fig.~\ref{fig:stability_exp2}.  
We first consider $d_{k^L}(F,F')$. As we have seen in the previous experiment, 
$d_{k^L}$ will be dominated 
by variations in the points of high persistence. Similarly,  $d_{k^L}(F,G)$
will also be dominated by these points as long 
as $\lambda$ is sufficiently large. Hence, instances of classes A and B 
would be inseparable in a nearest neighbor setup. In contrast, $d_{B}$, $d_{W,p}$
and $d_{k_\sigma}$ do \textit{not} over-emphasize points of high 
persistence and thus allow to distinguish classes A and B.
\begin{figure}[t!]
\centering
\includegraphics[width=0.98\columnwidth]{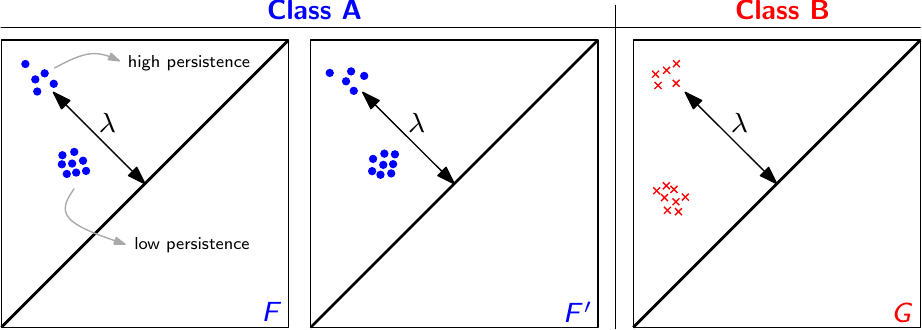}
\caption{Two persistence diagrams \textcolor{blue}{$F,F'$} from \textcolor{blue}{class A} 
and one diagram \textcolor{red}{$G$} from \textcolor{red}{class B}.
The classes only differ in their points of low-persistence (\ie, points closer to 
the diagonal).\label{fig:stability_exp2}}
\end{figure}

\begin{figure*}
\centering \includegraphics[width=1\textwidth]{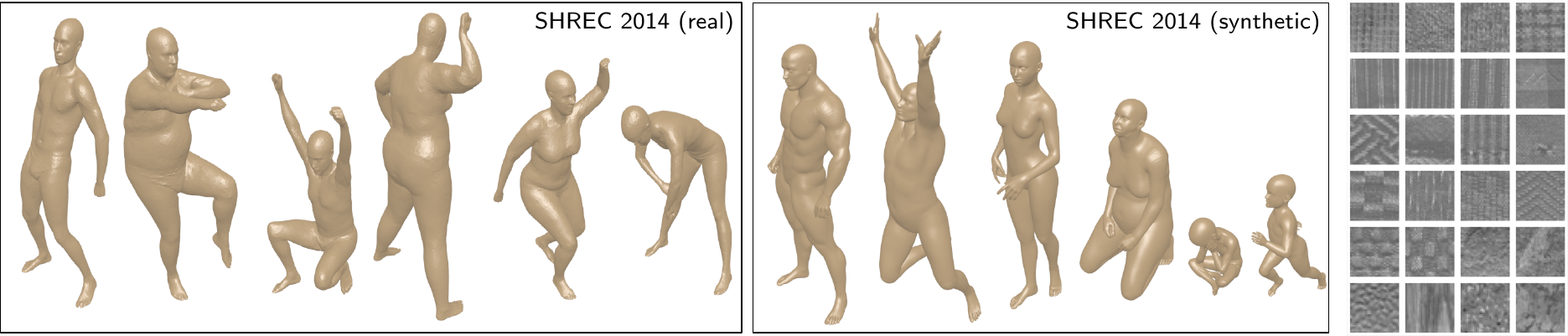}
\caption{Examples from \textsc{SHREC 2014} \cite{Pickup2014} (left, middle) and OuTeX \texttt{Outex\_TC\_00000} \cite{Ojala02a} (right).\label{fig:dataset_visual}}
\end{figure*}

\subsection{Empirical results}
\label{subsection:empirical_results}

We report results on two vision tasks where persistent homology
has already been shown to provide valuable discriminative information~\cite{Li14a}:
\textit{shape classification/retrieval}
and \textit{texture image classification}. The purpose of the experiments 
is \textit{not} to outperform the state-of-the-art on these problems -- 
which would be rather challenging by exclusively using topological information -- 
but to demonstrate the advantages of $k_\sigma$ and $d_{k_\sigma}$ over 
$k^L$ and $d_{k^L}$.

\vspace{-0.3cm}
\paragraph{Datasets.}
For shape classification/retrieval, we use the \textsc{SHREC 2014} \cite{Pickup2014} benchmark, 
see Fig.~\ref{fig:dataset_visual}. It consists of both \textit{synthetic} and \textit{real} shapes, 
given as 3D meshes. The synthetic part of the data contains $300$ 
meshes of humans (five males, five females, five children) in $20$ different 
poses; the real part contains $400$ meshes from $40$ humans (male, female) 
in $10$ different poses. We use the meshes in full resolution, \ie, 
without any mesh decimation. For classification, the objective is 
to distinguish between the different human models, \ie, a 15-class
problem for SHREC 2014 (synthetic) and a 40-class problem for SHREC 2014
(real).

For texture recognition, we use the \texttt{Outex\_TC\_00000} 
benchmark \cite{Ojala02a}, downsampled to $32\times 32$ pixel 
images. The benchmark provides 100 predefined training/testing splits 
and each of the 24 classes is equally represented by 10 images during 
training and testing.

\paragraph{Implementation.} 
For shape classification/retrieval, we compute the classic \textit{Heat 
Kernel Signature (HKS)} \cite{Sun09a} over a range of ten time 
parameters $t_i$ of increasing value. For each specific choice of 
$t_i$, we obtain a piecewise linear function on the surface mesh of 
each object. As discussed in Sec.~\ref{section:background},
we then compute the persistence diagrams of the induced filtrations
in dimensions $0$ and $1$.

For texture classification, we compute CLBP \cite{Guo10a} descriptors,
(\cf~\cite{Li14a}). Results are reported for the rotation-invariant 
versions of the CLBP-Single (\texttt{CLBP-S}) and the CLBP-Magnitude 
(\texttt{CLBP-M}) operator with $P=8$ neighbours and radius $R=1$.
Both operators produce a scalar-valued response image which 
can be interpreted as a weighted cubical cell complex and its lower star
filtration is used to compute persistence diagrams; see \cite{Wagner12a} 
for details.

For both types of input data, the persistence diagrams
are obtained using \textsc{Dipha}~\cite{Bauer14a},
which can directly handle meshes and images. A 
standard soft margin $C$-SVM classifier \cite{Scholkopf01}, as implemented in 
\textsc{Libsvm} \cite{Chang11a}, is used for classification. The 
cost factor $C$ is tuned using ten-fold cross-validation 
on the training data. For the kernel $k_\sigma$, this cross-validation 
further includes the kernel scale $\sigma$.

\vspace{-0.2cm}
\subsubsection{Shape classification}
\label{subsubsection:shape_classification}

Tables \ref{table:shrec14_clf_syn_inner_product} and 
\ref{table:shrec14_clf_real_inner_product}
list the classification results for $k_\sigma$ and 
$k^L$ on \textsc{SHREC 2014}. All results are averaged over 
ten cross-validation runs using random 70/30 training/testing 
splits with a roughly equal class distribution. We report 
results for $1$-dimensional features only; $0$-dimensional
features lead to comparable performance.

On both real and synthetic data, we observe that $k_\sigma$ 
leads to consistent improvements over $k^L$. For some choices of~$t_i$, 
the gains even range up to $30\%$, while in other cases, 
the improvements are relatively small. This can be explained by the 
fact that varying the HKS time $t_i$ essentially varies the 
smoothness of the input data. 
The scale $\sigma$ in $k_\sigma$ allows to compensate---at the classification stage---for 
unfavorable smoothness settings to a certain extent, see Sec.~\ref{section:kernel}.
In contrast, $k^L$ does not have this capability and essentially relies on suitably 
preprocessed input data. For some choices of $t_i$, $k^L$ does in fact lead to classification 
accuracies close to $k_\sigma$. However, when using $k^L$, we have to carefully adjust the 
HKS time parameter, corresponding to changes in the input data.
This is undesirable in most situations, since HKS computation for meshes 
with a large number of vertices can be quite time-consuming and sometimes we might not even have access 
to the meshes directly. The improved classification rates for $k_\sigma$ indicate that 
using the additional degree of freedom is in fact beneficial for performance.

\begin{table}[t!]
\footnotesize
\centering{
\begin{tabular}{|c|c|c||r|}
\hline
HKS $t_i$	& $k^L$ & $k_\sigma$ & \multicolumn{1}{c|}{$\Delta$} \\
\hline
$t_{1}$	 & $68.0 \pm 3.2$ & $94.7 \pm 5.1$ & \cellcolor{green!10}{$+26.7$}\\
$t_{2}$	 & $\mathbf{88.3} \pm 3.3$ & $\mathbf{99.3} \pm 0.9$ & \cellcolor{green!10}{$+11.0$}\\
$t_{3}$	 & $61.7 \pm 3.1$ & $96.3 \pm 2.2$ & \cellcolor{green!10}{$+34.7$}\\
$t_{4}$	 & $81.0 \pm 6.5$ & $97.3 \pm 1.9$ & \cellcolor{green!10}{$+16.3$}\\
$t_{5}$	 & $84.7 \pm 1.8$ & $96.3 \pm 2.5$ & \cellcolor{green!10}{$+11.7$}\\
$t_{6}$	 & $70.0 \pm 7.0$ & $93.7 \pm 3.2$ & \cellcolor{green!10}{$+23.7$}\\
$t_{7}$	 & $73.0 \pm 9.5$ & $88.0 \pm 4.5$ & \cellcolor{green!10}{$+15.0$}\\
$t_{8}$	 & $81.0 \pm 3.8$ & $88.3 \pm 6.0$ & \cellcolor{green!10}{$+7.3$}\\
$t_{9}$	 & $67.3 \pm 7.4$ & $88.0 \pm 5.8$ & \cellcolor{green!10}{$+20.7$}\\
$t_{10}$ & $55.3 \pm 3.6$ & $91.0 \pm 4.0$ & \cellcolor{green!10}{$+35.7$}\\\hline
\end{tabular}}
\caption{\label{table:shrec14_clf_syn_inner_product}
Classification performance on SHREC 2014 (synthetic).}
\bigskip
\centering{
\begin{tabular}{|c|c|c||r|}
\hline
HKS $t_i$	& $k^L$  & $k_\sigma$  & \multicolumn{1}{c|}{$\Delta$} \\
\hline
$t_{1}$	 & $45.2 \pm 5.8$ & $48.8 \pm 4.9$ & \cellcolor{green!10}{$+3.5$}\\
$t_{2}$	 & $31.0 \pm 4.8$ & $46.5 \pm 5.3$ & \cellcolor{green!10}{$+15.5$}\\
$t_{3}$	 & $30.0 \pm 7.3$ & $37.8 \pm 8.2$ & \cellcolor{green!10}{$+7.8$}\\
$t_{4}$	 & $41.2 \pm 2.2$ & $50.2 \pm 5.4$ & \cellcolor{green!10}{$+9.0$}\\
$t_{5}$	 & $46.2 \pm 5.8$ & $62.5 \pm 2.0$ & \cellcolor{green!10}{$+16.2$}\\
$t_{6}$	 & $33.2 \pm 4.1$ & $58.0 \pm 4.0$ & \cellcolor{green!10}{$+24.7$}\\
$t_{7}$	 & $31.0 \pm 5.7$ & $\mathbf{62.7} \pm 4.6$ & \cellcolor{green!10}{$+31.7$}\\
$t_{8}$	 & $\mathbf{51.7} \pm 2.9$ & $57.5 \pm 4.2$ & \cellcolor{green!10}{$+5.8$}\\
$t_{9}$	 & $36.0 \pm 5.3$ & $41.2 \pm 4.9$ & \cellcolor{green!10}{$+5.2$}\\
$t_{10}$ & $2.8 \pm 0.6$ & $27.8 \pm 5.8$ & \cellcolor{green!10}{$+25.0$}\\\hline
\end{tabular}}
\caption{\label{table:shrec14_clf_real_inner_product} Classification performance on \textsc{SHREC 2014} (real).}
\end{table}

\vspace{-0.2cm}
\subsubsection{Shape retrieval}
\label{subsection:shape_retrieval}

In addition to the classification experiments, we report on shape retrieval performance 
using standard evaluation measures (see~\cite{Shilane04a,Pickup2014}). This allows 
us to assess the behavior of the kernel-induced distances $d_{k_\sigma}$ and
$d_{k^L}$.

For brevity, only the nearest-neighbor performance is listed in Table 
\ref{table:shrec14_retrieval} (for a listing of all measures, see
Appendix~\ref{section:additionalresults}). Using each shape as a query shape once, nearest-neighbor performance measures 
how often the top-ranked shape in the retrieval result belongs to the same 
class as the query. To study the effect of tuning 
the scale $\sigma$, the column $d_{k_\sigma}$ lists the \textit{maximum} nearest-neighbor performance 
that can be achieved over a range of scales.

As we can see, the results are similar to the classification experiment. 
However, at a few specific settings of the HKS time $t_i$, $d_{k^L}$ 
performs on par, or better than $d_{k_\sigma}$. As noted in
Sec.~\ref{subsubsection:shape_classification}, this can be explained by the 
changes in the smoothness of the input data, induced by different HKS times
$t_i$. Another observation is that nearest-neighbor performance of $d_{k^L}$ is 
quite unstable around the top result with respect to $t_i$. For example, 
it drops at $t_2$ from 91\% to 53.3\% and 76.7\% on 
\textsc{SHREC 2014} (synthetic) and at $t_8$ from 70\% to 45.2\% and 
43.5\% on \textsc{SHREC 2014} (real). In contrast, $d_{k_\sigma}$
exhibits stable performance around the optimal $t_i$.

To put these results into context with existing works in shape retrieval, 
Table~\ref{table:shrec14_retrieval} also lists the top three entries (out of 22) 
of \cite{Pickup2014} on the same benchmark. On both real and 
synthetic data, $d_{k_\sigma}$ ranks among the top five entries.
This indicates that topological persistence alone is a rich source of 
discriminative information for this particular problem. In addition, since we only assess one
HKS time parameter at a time, performance could potentially be improved 
by more elaborate fusion strategies.

\begin{table}
\footnotesize
\centering{\begin{tabular}{|c|cc||r| c |cc||r|}
\cline{1-4}\cline{6-8}
HKS $t_i$ & \multicolumn{1}{c|}{$d_{k^L}$} & \multicolumn{1}{c||}{$d_{k_\sigma}$} & \multicolumn{1}{c|}{$\Delta$} &\hspace{-0.3cm} & \multicolumn{1}{c|}{$d_{k^L}$} & \multicolumn{1}{c||}{$d_{k_\sigma}$} & 
\multicolumn{1}{c|}{$\Delta$} \\
\hhline{-|---~|---}
$t_1$ & $53.3$ 		& $88.7$ &		\cellcolor{green!10}{$+35.4$} 			&\hspace{-0.3cm} & $24.0$ 	& $23.7$					& $\cellcolor{red!10}{-0.3}$ 		\\
$t_2$ & $\mathbf{91.0}$ 		& $\mathbf{94.7}$ &		$ \cellcolor{green!10}{+3.7}$ 		&\hspace{-0.3cm} & $20.5$ 				& $25.7$		& $\cellcolor{green!10}{+5.2}$		\\
$t_3$ & $76.7$ 		& $91.3$ &		$ \cellcolor{green!10}{+14.6}$			&\hspace{-0.3cm} & $16.0$ 				& $18.5$		& $\cellcolor{green!10}{+2.5}$	 	\\
$t_4$ & $84.3$ 		& $93.0$ &		$ \cellcolor{green!10}{+8.7}$ 			&\hspace{-0.3cm} & $26.8$ 				& $33.0$		& $\cellcolor{green!10}{+6.2}$		\\
$t_5$ & $85.0$ 		& $92.3$ &		$ \cellcolor{green!10}{+7.3}$ 			&\hspace{-0.3cm} & $28.0$ 				& $38.7$		& $\cellcolor{green!10}{+10.7}$ 	\\
$t_6$ & $63.0$ 		& $77.3$ &		$ \cellcolor{green!10}{+14.3}$ 			&\hspace{-0.3cm} & $28.7$ 				& $36.8$		& $\cellcolor{green!10}{+8.1}$		\\
$t_7$ & $65.0$ 		& $80.0$ &		$ \cellcolor{green!10}{+15.0}$ 			&\hspace{-0.3cm} & $43.5$ 				& $52.7$		& $\cellcolor{green!10}{+9.2}$		\\
$t_8$ & $73.3$ 		& $80.7$ &		$ \cellcolor{green!10}{+7.4}$ 			&\hspace{-0.3cm} & $\mathbf{70.0}$ 	& $\mathbf{58.2}$					& $\cellcolor{red!10}{-11.8}$		\\
$t_9$ & $73.0$ 		& $83.0$ &		$ \cellcolor{green!10}{+10.0}$ 			&\hspace{-0.3cm} & $45.2$ 				& $56.7$		& $\cellcolor{green!10}{+11.5}$ 	\\
$t_{10}$ & $51.3$	& $69.3$ &		$ \cellcolor{green!10}{+18.0}$ 			&\hspace{-0.3cm} & $3.5$   				& $44.0$		& $\cellcolor{green!10}{+40.5}	$	\\
\hhline{-|---~|---}
Top-$3$ \cite{Pickup2014} & \multicolumn{3}{c|}{$99.3$ -- $92.3$ -- $91.0$} &\hspace{-0.3cm} & \multicolumn{3}{c|}{$68.5$ -- $59.8$ -- $58.3$  }\\
\hhline{-|---~|---}
\end{tabular}}
\caption{\label{table:shrec14_retrieval}Nearest neighbor retrieval performance. \textit{Left:} \textsc{SHREC 2014}
(synthetic); \textit{Right:} \textsc{SHREC 2014} (real).}
\end{table}

\subsection{Texture recognition}
\label{subsection:texture_recognition}

For texture recognition, all results are averaged over the 
$100$ training/testing splits of the \texttt{Outex\_TC\_00000} 
benchmark. Table \ref{table:outex} lists the performance of
a SVM classifier using $k_\sigma$ and $k^L$ for $0$-dimensional 
features (\ie, connected components). 
Higher-dimensional features were not informative for this problem. For comparison, 
Table~\ref{table:outex} also lists the performance of a SVM, trained on 
normalized histograms of \texttt{CLBP-S/M} responses,
using a $\chi^2$ kernel.

First, from Table \ref{table:outex}, it is evident that $k_\sigma$
performs better than $k^L$ by a large margin, with gains up 
to $\approx$11\% in accuracy. Second, it is also apparent that,
for this problem, topological information alone 
is not competitive with SVMs using simple orderless operator response 
histograms. However, the results of \cite{Li14a} show that a \textit{combination} 
of persistence information (using persistence landscapes) with conventional 
bag-of-feature representations leads to 
state-of-the-art performance. While this indicates the complementary 
nature of topological features, it also suggests that kernel combinations (\eg, via multiple-kernel learning \cite{Gonen11a}) could lead to even greater gains by including the proposed 
kernel $k_\sigma$.

To assess the stability of the (customary) cross-validation strategy
to select a specific $\sigma$, Fig.~\ref{fig:acc_vs_scale} illustrates
classification performance as a function of the latter. Given the 
smoothness of the performance curve, it seems unlikely that 
parameter selection via cross-validation will be sensitive to a specific discretization of the search range $[\sigma_{\min},\sigma_{\max}]$.

Finally, we remark that tuning $k^L$ has the same 
drawbacks in this case as in the shape classification experiments. 
While, in principle, we could smooth the textures, the CLBP response images, 
or even tweak the radius of the CLBP operators, all those strategies 
would require changes at the beginning of the processing pipeline. 
In contrast, adjusting the scale $\sigma$ in $k_\sigma$ is done at the \textit{end} 
of the pipeline during classifier training.

\begin{table}[t!!]
\footnotesize
\centering 
\begin{tabular}{|r|c|c||c|}
\hline
CLBP Operator	&  $k^L$  & $k_\sigma$ & $\Delta$ \\
	\hline
\texttt{CLBP-S} 	& $58.0\pm 2.3$ & $\mathbf{69.2\pm 2.7}$ & \cellcolor{green!10}{$+11.2$} \\

\texttt{CLBP-M} 	& $45.2\pm 2.5$ & $\mathbf{55.1\pm 2.5}$ & \cellcolor{green!10}{$+9.9$} \\
\hline\hline
\texttt{CLBP-S} (SVM-$\chi^2$) & \multicolumn{3}{c|}{$76.1 \pm 2.2$}\\
\texttt{CLBP-M} (SVM-$\chi^2$) & \multicolumn{3}{c|}{$76.7 \pm 1.8$}\\
\hline
\end{tabular}
\caption{Classification performance on \texttt{Outex\_TC\_00000}.\label{table:outex}}
\end{table}

\section{Conclusion}
\label{section:conclusion}

\begin{figure}
\centering{
\includegraphics[width=0.475\columnwidth]{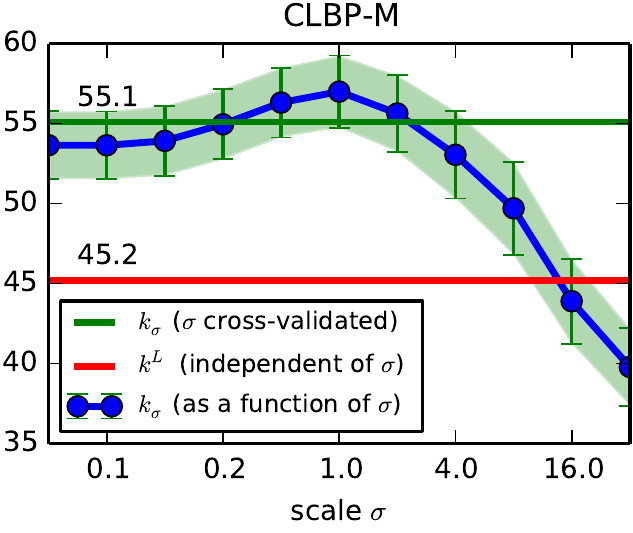}\hfill
\includegraphics[width=0.475\columnwidth]{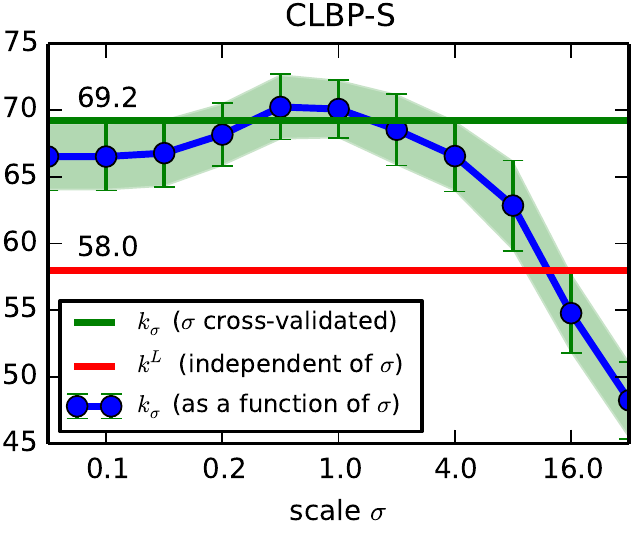}
\caption{\label{fig:acc_vs_scale}Texture classification performance of
a SVM classifier with (1) \textcolor{blue}{the kernel $k_\sigma$ as a function of $\sigma$}, (2) the \textcolor{darkgreen}{kernel $k_\sigma$ with $\sigma$ cross-validated} and (3) the \textcolor{red}{kernel 
$k^L$} are shown.}}
\end{figure}

We have shown, both theoretically and empirically, that the proposed kernel exhibits good behavior for tasks like shape classification or texture recognition using a SVM. Moreover, the ability to tune a scale parameter has proven beneficial in practice.

One possible direction for future work would be to address computational 
bottlenecks in order to enable application in large scale scenarios.  
This could include leveraging additivity and stability in order to 
approximate the value of the kernel within given error bounds, 
in particular, by reducing the number of distinct points in the 
summation of~\eqref{eqn:l2ip}.

While the 1-Wasserstein distance is well established and has proven useful in applications, 
we hope to improve the understanding of stability for persistence diagrams \wrt the Wasserstein distance beyond the previous estimates.
Such a result would extend the stability of our kernel from persistence diagrams to the underlying data, leading
to a full stability proof for topological machine learning.

In summary, our method enables the use of topological information 
in all kernel-based machine learning methods. It will therefore be 
interesting to see which other application areas will profit from 
topological machine learning.




{%
    \footnotesize
    \def\baselinestretch{0.9}
    \bibliographystyle{ieee}

}

\cleardoublepage
\appendix
\section*{Appendix}
\section{Indefiniteness of $d_{W,p}$}
\label{section:definiteness}

It is tempting to try to employ the Wasserstein distance 
for constructing a kernel on persistence diagrams.
For instance, in Euclidean space, $k(x,y) = -\|x - y\|^2, x,y 
\in \mathbb{R}^n$ is conditionally positive definite
and can be used within SVMs. Hence, the question 
arises if $k(x,y) = -d_{W,p}(x,y), x,y \in \mathcal{D}$ can 
be used as well.

In the following, we demonstrate (via counterexamples) that 
neither $-d_{W,p}$ nor $\exp(-\xi d_{W,p}(\cdot,\cdot))$~-- for
different choices of $p$~-- are (conditionally) 
positive definite. Thus, they cannot be employed in 
kernel-based learning techniques.

First, we briefly repeat some definitions to establish the terminology; 
this is done to avoid potential confusion, \wrt references \cite{Berg84a,Bapat97a,Scholkopf01}), 
about what is referred to as (conditionally) positive/negative definiteness 
in the context of kernel functions.

\begin{definition}
A symmetric matrix $\mathbf{A} \in \nR^{n \times n}$ is called positive definite (p.d.) if $\mathbf{c}^\top\mathbf{A} \mathbf{c} \ge 0$ for all $\mathbf{c} \in \nR^n$.
A symmetric matrix $\mathbf{A} \in \nR^{n \times n}$ is called negative definite (n.d.) if $\mathbf{c}^\top \mathbf{A} \mathbf{c} \le 0$ for all $\mathbf{c} \in \nR^n$.
\end{definition}

Note that in literature on linear algebra the notion of definiteness as introduced 
above is typically known as semidefiniteness. For the sake of brevity, in the 
kernel literature the prefix ``semi'' is typically dropped.

\begin{definition}
A symmetric matrix $\mathbf{A} \in \nR^{n \times n}$ is called conditionally positive definite (c.p.d.) if $\mathbf{c}^t \mathbf{A} \mathbf{c} \ge 0$ for all $\mathbf{c} = (c_1, \dots, c_n) \in \nR^n$ s.t.\ $\sum_i c_i = 0$.
A symmetric matrix $\mathbf{A} \in \nR^{n \times n}$ is called conditionally negative definite (c.n.d.) if $\mathbf{c}^\top \mathbf{A} c \le 0$ for all $\mathbf{c} = (c_1, \dots, c_n) \in \nR^n$ s.t.\ $\sum_i c_i = 0$.
\end{definition}

\begin{definition}
Given a set $\mathcal{X}$, a function $k \colon \mathcal{X} \times \mathcal{X} \to \nR$ is a \emph{positive definite kernel} if there exists a Hilbert space $\mathcal{H}$ and a map $\Phi \colon \mathcal{X} \to \mathcal{H}$ such that $k(x,y) = \langle \Phi(x), \Phi(y) \rangle_{\mathcal{H}}$.
\end{definition}

Typically a positive definite kernel is simply called \emph{kernel}. Roughly speaking, the utility 
of p.d.\@ kernels comes from the fact that they enable the ``kernel-trick'', \ie, the use of
algorithms that can be formulated in terms of dot products in an implicit feature space \cite{Scholkopf01}.
However, as shown by Sch\"olkopf in \cite{Schoelkopf01b}, this ``kernel-trick'' also works for 
distances, leading to the larger class of c.p.d. kernels (see Definition~\ref{def:cpdkernel}), which can be used in kernel-based
algorithms that are translation-invariant (\eg, SVMs or kernel PCA).

\begin{definition}
\label{def:cpdkernel}
A function $k \colon \mathcal{X} \times \mathcal{X} \to \nR$ is (conditionally) positive (negative, resp.) definite kernel if and only if $k$ is symmetric and for every finite subset $\{x_1, \dots, x_m\} \subseteq \mathcal{X}$ the Gram matrix $(k(x_i, x_j))_{i,j = 1, 1}^{m, m}$ is (conditionally) positive (negative, resp.) definite.
\end{definition}

To demonstrate that a function is not c.p.d.\@ or c.n.d., resp., we can look at the eigenvalues
of the corresponding Gram matrices. In fact, it is known that a matrix $\mathbf{A}$ is p.d.\@ 
if and only if all its eigenvalues are nonnegative. The following lemmas from \cite{Bapat97a} 
give similar, but weaker results for (nonnegative) c.n.d.\@ matrices, which will be useful to 
us.

\begin{lemma}[see Lemma 4.1.4 of \cite{Bapat97a}]
If $\mathbf{A}$ is a c.n.d.\@ matrix, then $\mathbf{A}$ has at most one positive 
eigenvalue.
\end{lemma}

\begin{corollary}[see Corollary 4.1.5 of \cite{Bapat97a}]
Let $\mathbf{A}$ be a nonnegative, nonzero matrix that is c.n.d. Then 
$\mathbf{A}$ has exactly one positive eigenvalue.
\label{cor:nonnegativecnd}
\end{corollary}

The following theorem establishes a relation between c.n.d.\@ and p.d.\@ kernels.

\begin{theorem}[see Chapter 2, \S2, Theorem 2.2 of \cite{Berg84a}]
Let $\mathcal{X}$ be a nonempty set and let $k: \mathcal{X} \times \mathcal{X} 
\to \mathbb{R}$ be symmetric. Then 
$k$ is a conditionally negative definite kernel if and only if $\exp(-\xi k(\cdot,\cdot))$ 
is a positive definite kernel for all $\xi >0$.
\label{thm:1}
\end{theorem}

In the code (\texttt{test\_negative\_type\_simple.m})\footnote{\url{https://gist.github.com/rkwitt/4c1e235d702718a492d3}; the file 
\texttt{options\_cvpr15.mat} can be found at: \url{http://www.rkwitt.org/media/files/options_cvpr15.mat}}, we generate simple examples for which the Gram matrix 
$\mathbf{A} = (d_{W,p}(x_i,x_j))_{i,j=1,1}^{m,m}$ -- for various choices of $p$ -- has at least two positive and two negative eigenvalue. Thus, it is neither (c.)n.d.\@ nor (c.)p.d. according to Corollary~\ref{cor:nonnegativecnd}. Consequently, the 
function $\exp(-d_{W,p})$ is not p.d.\@ either, by virtue of Theorem~\ref{thm:1}. 
To run the \textsc{Matlab} code, simply execute:
\begin{lstlisting}
load options_cvpr15.mat;
test_negative_type_simple(options);
\end{lstlisting}
This will generate a short summary of the eigenvalue computations for a selection 
of values for $p$, including $p=\infty$ (bottleneck distance).

\vspace{0.2cm}\noindent
\textbf{Remark.} While our simple counterexamples suggest that 
typical kernel constructions using $d_{W,p}$ for different $p$ 
(including $p=\infty$) do not lead to (c.)p.d. kernels, a formal
assessment of this question remains an open research question.

\section{Plots of the feature map $\Phi_\sigma$}
\label{section:featuremapplots}

Given a persistence diagram $D$, we consider the solution $u\colon \Omega \times
\mathbb{R}_{\geq 0} \rightarrow \mathbb{R}, (x,t) \mapsto u(x,t)$ of the
following partial differential equation
\begin{align*}
    \Delta_x u &= \partial_t u &&\text{in $\Omega \times \mathbb{R}_{> 0}$}, \\
    u &= 0 &&\text{on $\partial\Omega \times \mathbb{R}_{\geq 0}$}, \\
    u &= \sum_{p \in D} \delta_p &&\text{on $\Omega \times \{0\}$}.
\end{align*}
To solve the partial differential equation, we extend the domain from $\Omega$ to $\nR^2$ and
consider for each $p \in D$ a Dirac delta $\delta_p$ and a Dirac delta
$-\delta_{\overline{p}}$, as illustrated in Fig.~\ref{fig:plot3d-pre} (left). By convolving $\sum_{p
    \in D} \delta_p - \delta_{\overline{p}}$ with a Gaussian kernel, see 
    Fig.~\ref{fig:plot3d-pre} (right), we obtain a
solution $u\colon \nR^2 \times \mathbb{R}_{\geq 0} \rightarrow \mathbb{R}, (x,t)
\mapsto u(x,t)$ for the following partial differential equation:
\begin{align*}
    \Delta_x u &= \partial_t u &&\text{in $\nR^2 \times \mathbb{R}_{> 0}$}, \\
    u &= \sum_{p \in D} \delta_p - \delta_{\overline{p}} &&\text{on $\nR^2
        \times \{0\}$}.
\end{align*}
Restricting the solution $u$ to $\Omega \times \nR_{\ge 0}$, we then obtain
the following solution $u \colon \Omega \times \nR_{\ge 0} \to \nR$,
\begin{align}
    u(x, t) = \frac{1}{4\pi t} \sum_{p \in D} e^{-\frac{\|x - p\|^2}{4t}} -
    e^{-\frac{\|x - \overline{p}\|^2}{4t}}
\end{align}
for the original partial differential equation and $t > 0$. This 
yields the feature map $\Phi_\sigma \colon \mathcal{D} 
\to L_2(\Omega)$:
\begin{align}
    \Phi_\sigma(D) \colon \Omega \to \nR, \quad x \mapsto \frac{1}{4\pi \sigma}
    \sum_{p \in D} e^{-\frac{\|x - p\|^2}{4 \sigma}} - e^{-\frac{\|x -
            \overline{p}\|^2}{4 \sigma}} .
\end{align}

\begin{figure}[tbh]
    \centering
    \includegraphics[page=1, viewport = 0 0 195 130, clip=true,scale=0.5]{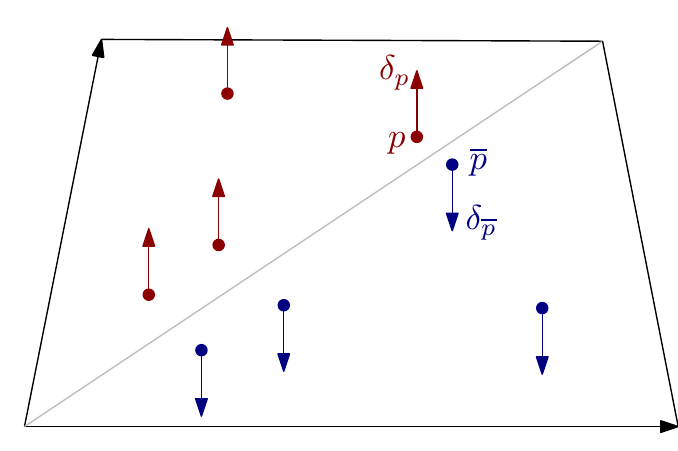}
    \hspace{1em}
    \includegraphics[page=2, viewport = 0 0 195 130, clip=true,scale=0.5]{plot3d-pre}
    \caption{\label{fig:plot3d-pre}
        Solving the partial differential equation: First (left), we extend
        the domain from $\Omega$ to $\nR^2$ and consider for each $p \in D$ a Dirac delta
        $\delta_p$ (\textcolor{myred}{red}) and a Dirac delta $-\delta_{\overline{p}}$ (\textcolor{myblue}{blue}). Next
        (right), we convolve $\sum_{p \in D} \delta_p - \delta_{\overline{p}}$
        with a Gaussian kernel.}
\end{figure}

In Fig.~\ref{fig:featuremapsigma}, we illustrate the effect of an increasing
scale $\sigma$ on the feature map $\Phi_\sigma(D)$. Note that in the right plot
the influence of the low-persistence point close to the diagonal basically
vanishes. This effect is essentially due to the Dirichlet boundary condition and
is responsible for gaining stability for our persistence scale-space kernel
$k_\sigma$.

\begin{figure}[tbh]
    \centering
    \includegraphics[width=.32\columnwidth, viewport = 30 80 610 510, clip=true]{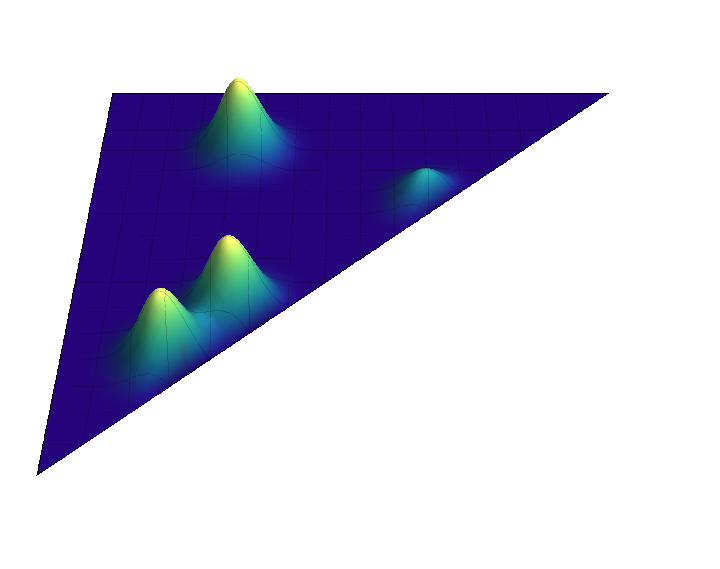}
    \includegraphics[width=.32\columnwidth, viewport = 30 80 610 510, clip=true]{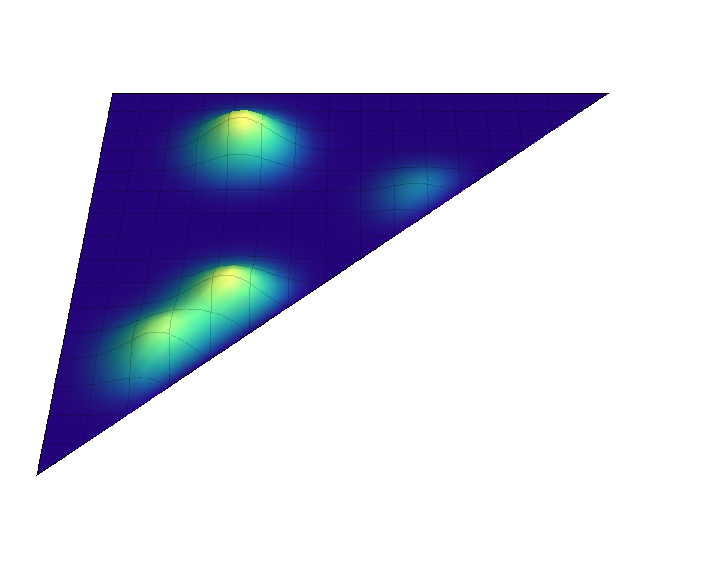}
    \includegraphics[width=.32\columnwidth, viewport = 30 80 610 510, clip=true]{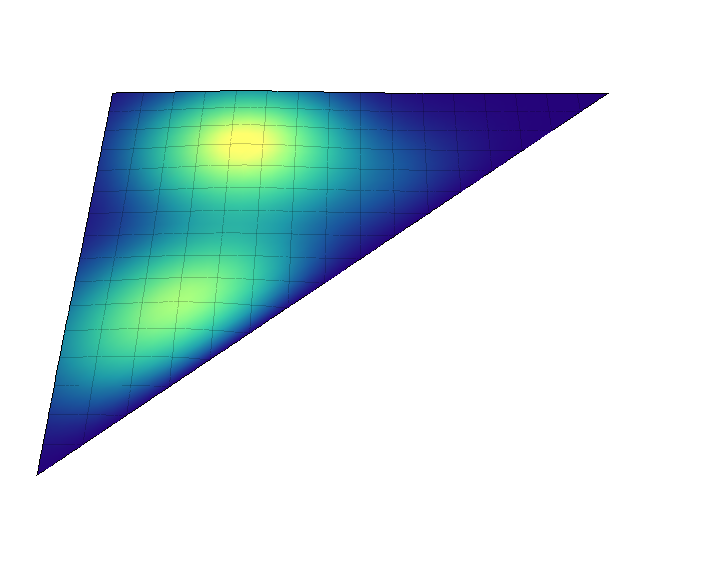}
    \caption{\label{fig:featuremapsigma}An illustration of the feature map
        $\Phi_\sigma(D)$ as a function in $L_2(\Omega)$ at growing scales
        $\sigma$ (from left to right).}
\end{figure}

\section{Closed-form solution for $k_\sigma$}
\label{section:kclosedform}

For two persistence diagrams $F$ and $G$, the persistence scale-space kernel
$k_\sigma(F, G)$ is 
defined as $\langle \Phi_\sigma(F),\Phi_\sigma(G)
\rangle_{L_2(\Omega)}$, which is
\begin{align*}
k_\sigma(F,G)
&= \int_{\Omega} \Phi_\sigma(F) \, \Phi_\sigma(G) \,dx.
\end{align*}
%
By
extending its domain from $\Omega$ to $\nR^2$, we see that $\Phi_\sigma(D)(x) =
- \Phi_\sigma(D)(\overline{x})$ for all $x \in \nR^2$. Hence, $\Phi_\sigma(F)(x)
\cdot \Phi_\sigma(G)(x) = \Phi_\sigma(F)(\overline{x}) \cdot
\Phi_\sigma(G)(\overline{x})$ for all $x \in \nR^2$, and we obtain
\begin{align*}
k_\sigma(F,G)
&= \frac{1}{2} \int_{\nR^2} \Phi_\sigma(F) \, \Phi_\sigma(G) \,dx \\
&= \frac{1}{2} \frac{1}{(4 \pi \sigma)^2} \int_{\nR^2}
    \left( \sum_{p \in F} e^{-\frac{\|x - p\|^2}{4 \sigma}} -
    e^{-\frac{\|x - \overline{p}\|^2}{4 \sigma}} \right)
    \cdot\\
    &\quad\left( \sum_{q \in G} e^{-\frac{\|x - q\|^2}{4 \sigma}} -
    e^{-\frac{\|x - \overline{q}\|^2}{4 \sigma}} \right)
     \,dx \\
&= \frac{1}{2} \frac{1}{(4 \pi \sigma)^2} \sum_{\substack{p \in F\\ q \in G}} \int_{\nR^2}
    \left( e^{-\frac{\|x - p\|^2}{4 \sigma}} -
    e^{-\frac{\|x - \overline{p}\|^2}{4 \sigma}} \right)
    \cdot\\
    &\quad \left( e^{-\frac{\|x - q\|^2}{4 \sigma}} -
    e^{-\frac{\|x - \overline{q}\|^2}{4 \sigma}} \right)
     \,dx \\
&= \frac{1}{(4 \pi \sigma)^2} \sum_{\substack{p \in F\\ q \in G}} \int_{\nR^2}
    e^{-\frac{\|x - p\|^2 + \|x - q\|^2}{4 \sigma}} - e^{-\frac{\|x - p\|^2 + \|x -
            \overline{q}\|^2}{4 \sigma}} \,dx.
\end{align*}
We calculate the integrals as follows:
\begin{align*}
    \int_{\nR^2} e^{-\frac{\|x - p\|^2 + \|x - q\|^2}{4 \sigma}} \, dx
    &= \int_{\nR^2} e^{-\frac{\|x - (p-q)\|^2 + \|x \|^2}{4 \sigma}} \, dx \\
    &=  \int_{\nR} \int_{\nR} e^{-\frac{(x_1 - \|p-q\|)^2 + x_2^2 \;+\; x_1^2 +
            x_2^2}{4 \sigma}} \, dx_1\, dx_2 \\
    &= \int_{\nR} e^{-\frac{x_2^2}{2 \sigma}} \, dx_2 \cdot
        \int_{\nR} e^{-\frac{(x_1 - \|p-q\|)^2 + x_1^2}{4 \sigma}} \, dx_1 \\
    &= \sqrt{2 \pi \sigma} \cdot
        \int_{\nR} e^{-\frac{(x_1 - \|p-q\|)^2 + x_1^2}{4 \sigma}} \, dx_1 \\
    &= \sqrt{2 \pi \sigma} \cdot
        \int_{\nR} e^{-\frac{(2 x_1 - \|p-q\|)^2 + \|p-q\|^2}{8 \sigma}} \, dx_1 \\
    &= \sqrt{2 \pi \sigma} \; e^{-\frac{\|p-q\|^2}{8 \sigma}} \cdot
        \int_{\nR} e^{-\frac{(2 x_1 - \|p-q\|)^2 }{8 \sigma}} \, dx_1 \\
    &= \sqrt{2 \pi \sigma} \; e^{-\frac{\|p-q\|^2}{8 \sigma}} \cdot
        \int_{\nR} e^{-\frac{x_1^2 }{2\sigma}} \, dx_1 \\
    &= 2 \pi \sigma \; e^{-\frac{\|p-q\|^2}{8\sigma}}.
\end{align*}
In the first step, we applied a coordinate transform that moves $x-q$ to $x$. In
the second step, we performed a rotation such that $p-q$ lands on the positive
$x_1$-axis at distance $\|p-q\|$ to the origin and we applied Fubini's theorem.
We finally obtain the closed-form expression for the kernel $k_\sigma$ as:
\begin{align*}
k_\sigma(F,G)
&= \frac{1}{(4 \pi \sigma)^2} \, 2 \pi \sigma \sum_{\substack{p \in F\\ q \in G}}
e^{-\frac{\|p-q\|^2}{8\sigma}} - e^{-\frac{\|p-\overline{q}\|^2}{8\sigma}}\\
&= \frac{1}{8 \pi \sigma} \sum_{\substack{p \in F\\ q \in G}}
e^{-\frac{\|p-q\|^2}{8\sigma}} - e^{-\frac{\|p-\overline{q}\|^2}{8\sigma}} .
\end{align*}

\newpage
\section{Additional retrieval results on SHREC 2014}
\label{section:additionalresults}
\begin{table}[h!]
\scriptsize
\begin{center}
\begin{tabular}{|c|cc||r| c |cc||r|}
\cline{1-4}\cline{6-8}
HKS $t_i$ & \multicolumn{1}{c|}{$d_{k^L}$} & \multicolumn{1}{c||}{$d_{k_\sigma}$} & \multicolumn{1}{c|}{$\Delta$} &\hspace{-0.3cm} & \multicolumn{1}{c|}{$d_{k^L}$} & \multicolumn{1}{c||}{$d_{k_\sigma}$} & 
\multicolumn{1}{c|}{$\Delta$} \\
\hhline{-|---~|---}
$t_1$ & $59.9$ 	& $71.3$ & $\cellcolor{green!10}{+11.4}$			&\hspace{-0.3cm} & $26.0$ & $21.4$ & $\cellcolor{red!10}{-4.6}$\\
$t_2$ & $\mathbf{75.1}$ 	& $76.0$ & $\cellcolor{green!10}{+0.9}$		&\hspace{-0.3cm} & $23.8$ & $22.7$ & $\cellcolor{red!10}{-1.1}$\\
$t_3$ & $49.6$ 	& $64.8$ & $\cellcolor{green!10}{+15.2}$			&\hspace{-0.3cm} & $19.1$ & $20.7$ & $\cellcolor{green!10}{+1.6}$\\
$t_4$ & $59.4$ 	& $\mathbf{77.5}$ & $\cellcolor{green!10}{+18.1}$	&\hspace{-0.3cm} & $23.5$ & $26.1$ & $\cellcolor{green!10}{+2.6}$\\
$t_5$ & $68.1$ 	& $75.2$ & $\cellcolor{green!10}{+7.1}$				&\hspace{-0.3cm} & $22.7$ & $27.4$ & $\cellcolor{green!10}{+4.7}$\\
$t_6$ & $50.0$ 	& $55.2$ & $\cellcolor{green!10}{+5.2}$				&\hspace{-0.3cm} & $18.9$ & $26.2$ & $\cellcolor{green!10}{+7.3}$\\
$t_7$ & $47.6$ 	& $53.6$ & $\cellcolor{green!10}{+6.0}$				&\hspace{-0.3cm} & $27.4$ & $31.8$ & $\cellcolor{green!10}{+4.4}$\\
$t_8$ & $53.1$ 	& $62.4$ & $\cellcolor{green!10}{+9.3}$				&\hspace{-0.3cm} & $\mathbf{45.3}$ & $\mathbf{39.8}$ & $\cellcolor{red!10}{-5.5}$\\
$t_9$ & $51.2$ 	& $56.3$ & $\cellcolor{green!10}{+5.1}$				&\hspace{-0.3cm} & $24.4$ & $30.3$ & $\cellcolor{green!10}{+5.9}$\\
$t_{10}$ & $39.6$ 	& $49.7$ & $\cellcolor{green!10}{+10.1}$			&\hspace{-0.3cm} & $2.5$  & $21.8$ & $\cellcolor{green!10}{+19.3}$\\
\hhline{-|---~|---}
Top-$3$ \cite{Pickup2014} & \multicolumn{3}{c|}{$83.2$ -- $76.4$ -- $76.0$} &\hspace{-0.3cm} & \multicolumn{3}{c|}{$54.1$ -- $47.2$ -- $45.1$  }\\
\hhline{-|---~|---}
\end{tabular}
\end{center}
\caption{\label{table:shrec14_retrieval_t1}\textbf{T1} retrieval performance. \textit{Left:} \textsc{SHREC 2014}
(synthetic); \textit{Right:} \textsc{SHREC 2014} (real).}
\begin{center}
\begin{tabular}{|c|cc||r| c |cc||r|}
\cline{1-4}\cline{6-8}
HKS $t_i$ & \multicolumn{1}{c|}{$d_{k^L}$} & \multicolumn{1}{c||}{$d_{k_\sigma}$} & \multicolumn{1}{c|}{$\Delta$} &\hspace{-0.3cm} & \multicolumn{1}{c|}{$d_{k^L}$} & \multicolumn{1}{c||}{$d_{k_\sigma}$} & 
\multicolumn{1}{c|}{$\Delta$} \\
\hhline{-|---~|---}
$t_1$ & $87.7$ & $91.4$ & $\cellcolor{green!10}{+3.7}$		&\hspace{-0.3cm} & $41.5$ & $34.6$ & $\cellcolor{red!10}{-6.9}$\\
$t_2$ & $\mathbf{91.1}$ & $\mathbf{95.1}$ & $\cellcolor{green!10}{+4.0}$		&\hspace{-0.3cm} & $40.8$ & $37.1$ & $\cellcolor{red!10}{-3.7}$\\
$t_3$ & $70.4$ & $83.4$ & $\cellcolor{green!10}{+13.0}$		&\hspace{-0.3cm} & $36.5$ & $36.8$ & $\cellcolor{green!10}{+0.3}$\\
$t_4$ & $77.7$ & $93.6$ & $\cellcolor{green!10}{+15.9}$		&\hspace{-0.3cm} & $39.8$ & $43.4$ & $\cellcolor{green!10}{+3.6}$\\
$t_5$ & $90.8$ & $92.3$ & $\cellcolor{green!10}{+1.5}$		&\hspace{-0.3cm} & $35.1$ & $41.8$ & $\cellcolor{green!10}{+6.7}$\\
$t_6$ & $73.9$ & $75.4$ & $\cellcolor{green!10}{+1.5}$		&\hspace{-0.3cm} & $31.6$ & $40.2$ & $\cellcolor{green!10}{+8.6}$\\
$t_7$ & $70.6$ & $74.4$ & $\cellcolor{green!10}{+3.8}$		&\hspace{-0.3cm} & $38.6$ & $47.6$ & $\cellcolor{green!10}{+9.0}$\\
$t_8$ & $73.3$ & $79.3$ & $\cellcolor{green!10}{+6.0}$		&\hspace{-0.3cm} & $\mathbf{56.5}$ & $\mathbf{57.6}$ & $\cellcolor{green!10}{+1.1}$\\
$t_9$ & $72.7$ & $76.2$ & $\cellcolor{green!10}{+3.5}$		&\hspace{-0.3cm} & $31.8$ & $42.5$ & $\cellcolor{green!10}{+10.7}$\\
$t_{10}$ & $57.8$ & $66.6$ & $\cellcolor{green!10}{+8.8}$   	&\hspace{-0.3cm} & $4.8$  & $31.0$ & $\cellcolor{green!10}{+26.2}$\\
\hhline{-|---~|---}
Top-$3$ \cite{Pickup2014} & \multicolumn{3}{c|}{$98.7$ -- $97.1$ -- $94.9$} &\hspace{-0.3cm} & \multicolumn{3}{c|}{$74.2$ -- $65.9$ -- $65.7$  }\\
\hhline{-|---~|---}
\end{tabular}
\end{center}
\caption{\label{table:shrec14_retrieval_t2}\textbf{T2} retrieval performance. \textit{Left:} \textsc{SHREC 2014}
(synthetic); \textit{Right:} \textsc{SHREC 2014} (real).}
\begin{center}
\begin{tabular}{|c|cc||r| c |cc||r|}
\cline{1-4}\cline{6-8}
HKS $t_i$ & \multicolumn{1}{c|}{$d_{k^L}$} & \multicolumn{1}{c||}{$d_{k_\sigma}$} & \multicolumn{1}{c|}{$\Delta$} &\hspace{-0.3cm} & \multicolumn{1}{c|}{$d_{k^L}$} & \multicolumn{1}{c||}{$d_{k_\sigma}$} & 
\multicolumn{1}{c|}{$\Delta$} \\
\hhline{-|---~|---}
$t_1$ & $60.6$ & $65.3$ & $\cellcolor{green!10}{+4.7}$		&\hspace{-0.3cm} & $25.4$ & $22.8$ & $\cellcolor{red!10}{-2.6}$\\
$t_2$ & $\mathbf{65.0}$ & $67.4$ & $\cellcolor{green!10}{+2.4}$&\hspace{-0.3cm} & $25.0$ & $23.4$ & $\cellcolor{red!10}{-1.6}$\\
$t_3$ & $48.4$ & $58.8$ & $\cellcolor{green!10}{+10.4}$		&\hspace{-0.3cm} & $24.0$ & $24.0$ & $\cellcolor{green!10}{+0.0}$\\
$t_4$ & $55.2$ & $\mathbf{67.6}$ & $\cellcolor{green!10}{+12.4}$&\hspace{-0.3cm} & $25.3$ & $27.4$ & $\cellcolor{green!10}{+2.1}$\\
$t_5$ & $63.7$ & $66.2$ & $\cellcolor{green!10}{+2.5}$		&\hspace{-0.3cm} & $21.6$ & $25.2$ & $\cellcolor{green!10}{+3.6}$\\
$t_6$ & $51.0$ & $52.7$ & $\cellcolor{green!10}{+1.7}$		&\hspace{-0.3cm} & $20.7$ & $23.7$ & $\cellcolor{green!10}{+3.0}$\\
$t_7$ & $48.4$ & $51.7$ & $\cellcolor{green!10}{+3.3}$		&\hspace{-0.3cm} & $22.5$ & $27.5$ & $\cellcolor{green!10}{+5.0}$\\
$t_8$ & $51.1$ & $56.5$ & $\cellcolor{green!10}{+5.4}$		&\hspace{-0.3cm} & $\mathbf{30.2}$ & $\mathbf{33.2}$ & $\cellcolor{green!10}{+3.0}$\\
$t_9$ & $50.4$ & $53.2$ & $\cellcolor{green!10}{+2.8}$		&\hspace{-0.3cm} & $15.8$ & $25.3$ & $\cellcolor{green!10}{+9.5}$\\
$t_{10}$ & $39.8$ & $46.7$ & $\cellcolor{green!10}{+6.9}$		&\hspace{-0.3cm} & $3.6$  & $19.0$ & $\cellcolor{green!10}{+15.4}$\\
\hhline{-|---~|---}
Top-$3$ \cite{Pickup2014} & \multicolumn{3}{c|}{$70.6$ -- $69.1$ -- $65.9$} &\hspace{-0.3cm} & \multicolumn{3}{c|}{$38.7$ -- $35.6$ -- $35.4$  }\\
\hhline{-|---~|---}
\end{tabular}
\end{center}
\caption{\label{table:shrec14_retrieval_em}\textbf{EM} retrieval performance. \textit{Left:} \textsc{SHREC 2014}
(synthetic); \textit{Right:} \textsc{SHREC 2014} (real).}
\begin{center}
\begin{tabular}{|c|cc||r| c |cc||r|}
\cline{1-4}\cline{6-8}
HKS $t_i$ & \multicolumn{1}{c|}{$d_{k^L}$} & \multicolumn{1}{c||}{$d_{k_\sigma}$} & \multicolumn{1}{c|}{$\Delta$} &\hspace{-0.3cm} & \multicolumn{1}{c|}{$d_{k^L}$} & \multicolumn{1}{c||}{$d_{k_\sigma}$} & 
\multicolumn{1}{c|}{$\Delta$} \\
\hhline{-|---~|---}
$t_1$ & $81.3$ & $91.5$ & $\cellcolor{green!10}{+10.2}$			&\hspace{-0.3cm} & $53.0$ & $49.6$ & $\cellcolor{red!10}{-3.4}$\\
$t_2$ & $\mathbf{92.1}$ & $93.4$ & $\cellcolor{green!10}{+1.3}$	&\hspace{-0.3cm} & $51.1$ & $51.3$ & $\cellcolor{green!10}{+0.2}$\\
$t_3$ & $80.3$ & $89.3$ & $\cellcolor{green!10}{+9.0}$			&\hspace{-0.3cm} & $47.7$ & $48.4$ & $\cellcolor{green!10}{+0.7}$\\
$t_4$ & $85.0$ & $\mathbf{93.8}$ & $\cellcolor{green!10}{+8.8}$	&\hspace{-0.3cm} & $52.7$ & $55.5$ & $\cellcolor{green!10}{+2.8}$\\
$t_5$ & $89.0$ & $93.2$ & $\cellcolor{green!10}{+4.2}$			&\hspace{-0.3cm} & $51.2$ & $55.5$ & $\cellcolor{green!10}{+4.3}$\\
$t_6$ & $78.6$ & $82.5$ & $\cellcolor{green!10}{+3.9}$			&\hspace{-0.3cm} & $48.1$ & $54.2$ & $\cellcolor{green!10}{+6.1}$\\
$t_7$ & $77.2$ & $81.6$ & $\cellcolor{green!10}{+4.4}$			&\hspace{-0.3cm} & $55.7$ & $60.5$ & $\cellcolor{green!10}{+4.8}$\\
$t_8$ & $80.4$ & $86.3$ & $\cellcolor{green!10}{+5.9}$			&\hspace{-0.3cm} & $\mathbf{72.8}$ & $\mathbf{68.3}$ & $\cellcolor{red!10}{-4.5}$\\
$t_9$ & $79.7$ & $83.9$ & $\cellcolor{green!10}{+4.2}$			&\hspace{-0.3cm} & $50.4$ & $61.0$ & $\cellcolor{green!10}{+10.6}$\\
$t_{10}$ & $70.8$ & $78.9$ & $\cellcolor{green!10}{+8.1}$			&\hspace{-0.3cm} & $27.7$ & $51.3$ & $\cellcolor{green!10}{+23.6}$\\
\hhline{-|---~|---}
Top-$3$ \cite{Pickup2014} & \multicolumn{3}{c|}{$97.7$ -- $93.8$ -- $92.7$} &\hspace{-0.3cm} & \multicolumn{3}{c|}{$78.1$ -- $71.7$ -- $71.2$  }\\
\hhline{-|---~|---}
\end{tabular}
\end{center}
\caption{\label{table:shrec14_retrieval_dcg}\textbf{DCG} retrieval performance. \textit{Left:} \textsc{SHREC 2014}
(synthetic); \textit{Right:} \textsc{SHREC 2014} (real).}
\end{table}

\end{document}